\documentclass{article}
\newif\ifneuripsstyle
\IfFileExists{neurips_2026.sty}{
    \neuripsstyletrue
    \usepackage[preprint]{neurips_2026}
}{
    \neuripsstylefalse
    \usepackage[margin=1in]{geometry}
    \usepackage{times}
    \usepackage[numbers]{natbib}
}
\usepackage{amsmath,amssymb,amsthm}
\usepackage{booktabs}
\usepackage{array,tabularx,colortbl}
\usepackage{graphicx}
\usepackage{hyperref}
\usepackage{xcolor}
\usepackage{enumitem}
\usepackage{microtype}
\usepackage{bbm}
\usepackage{pifont}
\usepackage{multirow}
\usepackage{tikz}
\usepackage{pgfplots}
\pgfplotsset{compat=1.18}
\usetikzlibrary{positioning,arrows.meta,fit,backgrounds}

\newtheorem{proposition}{Proposition}

\definecolor{saBlue}{HTML}{0A66C2}
\definecolor{saGreen}{HTML}{059669}
\definecolor{saRed}{HTML}{D62728}
\definecolor{saNavy}{HTML}{1F2A44}
\definecolor{saGrpoBar}{HTML}{9AA3AF}
\definecolor{saOrange}{HTML}{DD7A33}
\definecolor{saLightBlue}{HTML}{E8F1FB}
\definecolor{saLightGreen}{HTML}{DFF0E8}
\definecolor{saLightGray}{HTML}{EEF2F7}
\definecolor{saPurple}{HTML}{7655A5}
\definecolor{saLightPurple}{HTML}{F2EDF7}
\definecolor{saRule}{HTML}{CAD2DD}
\tikzset{
  corebox/.style={rounded corners=3pt, draw=gray!55, fill=saLightGray, align=center, font=\footnotesize, inner sep=5pt},
  seg/.style={rounded corners=2pt, minimum width=1.05cm, minimum height=0.62cm, font=\scriptsize},
  mixbox/.style={rounded corners=3pt, draw=saGreen, fill=saLightGreen, align=center, font=\scriptsize, inner sep=3pt},
  flowarr/.style={-{Stealth[length=6pt]}, draw=gray!60, line width=1pt},
}

\newcommand{\method}{SHARPO}
\newcommand{\student}{\pi_{\theta}}          % student policy
\newcommand{\teacher}{\pi_{\bar{\theta}}}    % stale teacher = frozen copy of the same policy
\newcommand{\sign}{\operatorname{sign}}
\newcommand{\clip}{\operatorname{clip}}

\title{\method{}: Segment-Level Credit Assignment for Agentic Reinforcement Learning}

\author{
    Xinchen Du$^{1,2}$\thanks{Work done during an internship at LinkedIn Corporation.} \quad Zhengze Zhou$^{1}$ 
    \quad Wenhui Zhu$^{1}$ \\ \bfseries Han Yu$^{1}$ 
    \quad Sen Na$^{2}$ 
    \quad Rohit Jain$^{1}$
    \quad Alborz Geramifard$^{1}$ \\
    \vspace{0.2cm}
    $^{1}$LinkedIn Corporation \quad
    $^{2}$Georgia Institute of Technology
}

\date{}

\begin{document}
\maketitle

\begin{abstract}
Agentic reinforcement learning (RL) trains a large language model (LLM) to act over long, multi-step interactions. However, a single localized error can cause task failure, while trajectory-level rewards provide limited guidance for assigning credit to individual decisions. To address this limitation, we introduce Segment-level Hindsight Advantage Reweighting for Policy Optimization (\method{}), a credit-assignment mechanism that refines Group Relative Policy Optimization (GRPO) at the level of environment-facing segments. Inspired by the existing on-policy self-distillation (OPSD) method, \method{} computes teacher--student log-probability gaps within each segment and uses the resulting signal to compute a bounded multiplier on the GRPO advantage. This multiplier is shared by all tokens within the segment, allowing credit to vary across different segments. With Qwen2.5-7B-Instruct, \method{} outperforms existing baselines on the ALFWorld and WebShop benchmarks, including GRPO, SDAR, RLSD, and StepOPSD.
\end{abstract}

\section{Introduction}
\label{sec:intro}

Post-training plays a key role in adapting large language models (LLMs) to specific tasks and improving their performance in interactive environments \citep{ouyang2022training,jin2025search}.
Reinforcement learning (RL) supports this adaptation by optimizing agent behavior using feedback from task outcomes.
When these outcomes can be automatically verified, their correctness provides a direct reward signal for training, as demonstrated in mathematical reasoning and search-based question answering \citep{deepseek2025r1,jin2025search}.
Group Relative Policy Optimization (GRPO) offers an efficient approach to learning from such rewards without training a separate value network \citep{shao2024deepseekmath}.
For each task, GRPO samples a group of rollouts and normalizes their rewards within the group to estimate advantages.
Under outcome supervision, each rollout receives a single advantage that is shared by all of its generated tokens.

In multi-turn interactions, however, an agent must make a sequence of decisions before receiving feedback on the final outcome.
In this case, the rewards are sparse and are available only at the end of an episode, hence they provide little direct information about the contribution of each intermediate step to that outcome.
For example, a shopping agent may find a suitable product but select an incorrect option before completing the purchase.
If the resulting trajectory receives a negative advantage, both the useful search and the mistaken selection inherit the same negative coefficient.
The shared advantage therefore does not distinguish actions that helped accomplish the task from those that led to failure.
This motivates us to design a credit-assignment mechanism in RL that uses finer-grained signals to refine the credit assigned to individual decisions within a trajectory.

One source of such fine-grained signals is on-policy distillation (OPD), where a student learns from a teacher's token-level predictions on responses generated by the student itself \citep{agarwal2024onpolicy}.
This provides dense supervision along the student's own generation process, but typically relies on a separate, often larger, teacher model.
On-policy self-distillation (OPSD) removes this dependence by using the same LLM as both teacher and student under different conditioning contexts \citep{zhao2026opsd}.
The teacher receives privileged information, such as a reference solution, and guides the student on its own rollouts.
Therefore, OPSD provides dense training feedback without requiring an external teacher model.

Inspired by OPSD, our \method{} reuses a successful rollout from the same GRPO group to show the teacher how the task can be completed.
The teacher uses this additional information to evaluate the student's behavior and provide feedback for credit assignment within the trajectory.
Concretely, we instantiate the teacher as a copy of the agent's policy and compute teacher--student log-probability differences.
These differences indicate how strongly the teacher supports the sampled behavior relative to the student.
We then use this signal to adjust the magnitude of GRPO's outcome-derived advantages. To sum up, our method \method{} combines fine-grained guidance from self-distillation with task-level supervision from environmental rewards.

In an interactive task, a trajectory contains semantically meaningful spans, such as agent reasoning, tool calls, and actions.
We refer to these spans as interaction segments and use those as the units for refining credit assignment. For each selected segment, \method{} assigns a weight based on the teacher--student log-probability difference for that segment. This weight scales the GRPO advantage uniformly within its segment, refining how the trajectory-level reward outcome is credited to intermediate substeps.
By incorporating the teacher's assessment of each substep, \method{} provides local guidance on the strength of reinforcement or penalization, alleviating the credit-assignment problem caused by sparse terminal rewards.

This segment-level formulation distinguishes \method{} from related approaches in how teacher feedback is organized and used for credit assignment. Unlike RLSD \citep{yang2026rlsd}, which reweights token-level advantages using self-divergence, or SDAR \citep{lu2026sdar}, which applies soft token-level gating to an auxiliary distillation loss over the full response, \method{} aggregates teacher feedback within selected interaction segments and assigns a shared advantage weight to each segment. Compared to StepOPSD \citep{zhang2026stepopsd}, which retains token-specific weights within selected action-centered spans, \method{} aggregates the feedback within each segment before determining its weight. Favorable and unfavorable evidence across a substep is therefore considered jointly, allowing the credit adjustment to reflect the teacher's overall support for the complete decision.

We evaluate \method{} on ALFWorld and WebShop using Qwen2.5-7B-Instruct.
As shown in Figure~\ref{fig:core}, \method{} achieves higher final validation success rates than baselines on both benchmarks at both parameter configurations. To summarize, our main contributions are threefold:

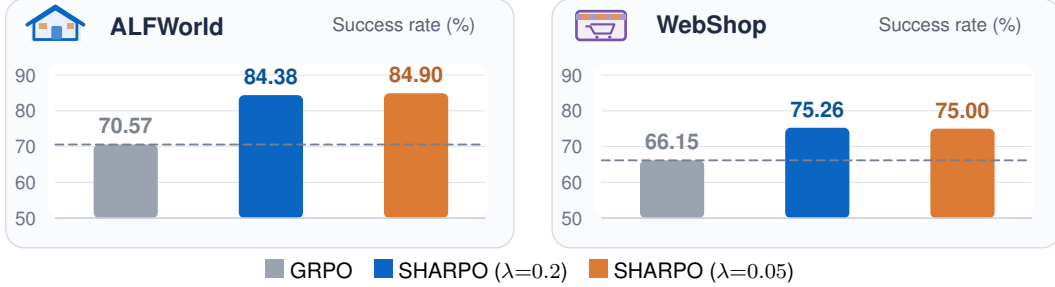
\begin{figure}[t]
\hypersetup{hidelinks}
\centering
\resizebox{\linewidth}{!}{%
\begin{tikzpicture}[x=1cm,y=.78cm,font=\sffamily\footnotesize,text=saNavy,
  line cap=round,line join=round]
\foreach \dx/\bench/\g/\b/\o in {0/ALFWorld/70.57/84.38/84.90,7.15/WebShop/66.15/75.26/75.00}{
 \begin{scope}[xshift=\dx cm]
  \filldraw[fill=saLightGray!30,draw=saRule!75,rounded corners=7pt,line width=.55pt] (0,0) rectangle (6.65,4.18);
  \fill[white,rounded corners=5pt] (.55,.48) rectangle (6.31,3.09);
  \node[anchor=west,font=\sffamily\small\bfseries] at (1.24,3.72) {\bench};
  \node[anchor=east,font=\sffamily\scriptsize,text=saNavy!75] at (6.27,3.72) {Success rate (\%)};
  \foreach \v in {50,60,70,80,90}{
   \pgfmathsetmacro{\yy}{.5+(\v-50)*.06}
   \draw[saRule!65,line width=.35pt] (.65,\yy) -- (6.27,\yy);
   \node[anchor=east,font=\sffamily\scriptsize,text=saNavy!65] at (.52,\yy) {\v};
  }
  \foreach \xx/\val/\col in {1.57/\g/saGrpoBar,3.47/\b/saBlue,5.37/\o/saOrange}{
   \pgfmathsetmacro{\hh}{.5+(\val-50)*.06}
   \fill[\col,rounded corners=2pt] ({\xx-.42},.5) rectangle ({\xx+.42},\hh);
   \node[anchor=south,font=\sffamily\footnotesize\bfseries,text=\col!80!black,inner sep=2pt] at (\xx,{\hh+.08}) {\val};
  }
  \draw[saNavy!60,densely dashed,line width=.7pt] (.65,{.5+(\g-50)*.06}) -- (6.27,{.5+(\g-50)*.06});
  \draw[saRule,line width=.6pt] (.65,.5) -- (6.27,.5);
 \end{scope}
}
% Benchmark pictograms sit outside the quantitative plotting area.
\begin{scope}[shift={(.65,3.73)},scale=.67]
 \fill[saLightBlue] (-.47,-.4) rectangle (.49,.2);
 \draw[saBlue,line width=1.05pt,fill=saLightBlue] (-.42,-.35) -- (-.42,.08) -- (0,.42) -- (.42,.08) -- (.42,-.35) -- cycle;
 \draw[saBlue,line width=1.4pt] (-.55,.02) -- (0,.48) -- (.55,.02);
 \fill[saBlue!75,rounded corners=.7pt] (-.08,-.35) rectangle (.13,-.01);
 \fill[saOrange!70] (-.31,-.07) rectangle (-.17,.07);
 \fill[saOrange!70] (.23,-.07) rectangle (.35,.07);
\end{scope}
\begin{scope}[shift={(7.79,3.73)},scale=.67]
 \draw[saPurple,line width=1pt,fill=saLightPurple,rounded corners=1.4pt] (-.48,-.36) rectangle (.48,.36);
 \fill[saPurple!60] (-.48,.14) rectangle (.48,.27);
 \foreach \xx in {-.38,-.08,.22}{\fill[saOrange!80] (\xx,.14) rectangle ++(.15,.13);}
 \draw[saPurple,line width=.8pt] (-.29,-.02) -- (-.21,-.02) -- (-.13,-.2) -- (.19,-.2) -- (.26,.03) -- (-.19,.03);
 \fill[saPurple] (-.08,-.29) circle (.035);\fill[saPurple] (.17,-.29) circle (.035);
\end{scope}
\end{tikzpicture}}
\vspace{4pt}
{\footnotesize\sffamily
\textcolor{saGrpoBar}{\rule{7pt}{7pt}}~GRPO\quad
\textcolor{saBlue}{\rule{7pt}{7pt}}~\method{} ($\lambda{=}0.2$)\quad
\textcolor{saOrange}{\rule{7pt}{7pt}}~\method{} ($\lambda{=}0.05$)
}
\caption{\textbf{Final validation success (\%) with Qwen2.5-7B-Instruct.} The dashed line marks GRPO. \method{} improves on GRPO at both mixing weights on both benchmarks, reaching $84.90\%$ on ALFWorld ($+14.32$ points) and $75.26\%$ on WebShop ($+9.11$ points). Bars are means over three independent runs at training step 200; the vertical axis starts at $50\%$. Table~\ref{tab:main-results} reports all baselines with standard deviations.}
\label{fig:core}
\end{figure}

\begin{itemize}
    \item We propose \method{}, a segment-level credit-assignment method that uses successful peer rollouts to guide the weighting of intermediate decisions in GRPO.
    By aggregating self-distillation feedback over meaningful interaction segments, it supplements sparse outcome supervision without requiring an external teacher or additional environment rollouts.

    \item We provide a theoretical explanation of why this self-distillation feedback is useful in our method and how it supports finer-grained credit assignment.

    \item We demonstrate that \method{} outperforms GRPO, RLSD, SDAR, and StepOPSD on both ALFWorld and WebShop.
    Analyses of subtask success rates and product-matching quality further characterize these gains.
\end{itemize}

\section{Method}
\label{sec:method}

We first review GRPO and OPSD in Section~\ref{sec:background} and then present \method{} in Section~\ref{sec:algorithm}.
Section~\ref{sec:theory} provides a theoretical analysis of our segment-level credit assignment.

\subsection{Preliminaries: GRPO and OPSD}
\label{sec:background}

We consider an LLM agent initialized from a pretrained model, with trainable parameters $\theta$.
Its autoregressive policy $\student$ defines a probability distribution over the next token given the available context.
Starting from a task prompt $x$, the agent generates responses and receives observations from the environment after executing actions.
Together, these responses and observations form an interaction trajectory $\tau$.
Let $y_\ell$ denote the $\ell$-th token generated by the agent across $\tau$, and let $T$ be the total number of such tokens.
Each token is sampled as $y_\ell \sim \student(\cdot \mid c_\ell)$, where $c_\ell$ contains the prompt $x$, previously generated tokens, and environment observations available before $y_\ell$.
At the end of the episode, a verifier assigns $r(\tau)=1$ if the task succeeds and $0$ otherwise.
We use RL post-training methods to update the policy parameters $\theta$ so as to maximize the expected terminal reward.

\paragraph{Group Relative Policy Optimization (GRPO).}
GRPO \citep{shao2024deepseekmath} samples $G$ trajectories $\{\tau_i\}_{i = 1}^G$ for the same task using the behavior policy $\pi_{\theta_{\mathrm{old}}}$, where $\theta_{\mathrm{old}}$
denotes the parameters saved before sampling and held fixed during subsequent updates. For rewards $r_i=r(\tau_i)$, it computes the trajectory advantage
\begin{equation}
\label{eq:grpo-adv}
A_i =
\frac{r_i-\operatorname{mean}(r_1,\ldots,r_G)}
{\operatorname{std}(r_1,\ldots,r_G)+\epsilon}.
\end{equation}
Here $\epsilon>0$ is a numerical stabilizer, and
the advantage $A_i$ measures whether a trajectory performs better or worse than its group average and is shared by all $T_i$ generated tokens in that trajectory. We write
$\omega_{i,\ell}(\theta)=
\student(y_{i,\ell}\mid c_{i,\ell})/
\pi_{\theta_{\mathrm{old}}}(y_{i,\ell}\mid c_{i,\ell})$
for the probability ratio between the current and behavior
policies. Hence we have
\begin{equation}
\label{eq:token-surrogate}
\ell_{i,\ell}(\theta;A)
=
\min\!\left\{
\omega_{i,\ell}(\theta)A,\;
\clip\!\big(
\omega_{i,\ell}(\theta),1-\varepsilon,1+\varepsilon
\big)A
\right\},
\end{equation}
where the $\clip$ function truncates the ratio to
$[1-\varepsilon,1+\varepsilon]$. Finally,
GRPO maximizes
\begin{equation}
\label{eq:grpo-obj}
\mathcal J_{\mathrm{GRPO}}(\theta)
=
\mathbb E\!\left[
\frac1G\sum_{i=1}^G
\frac1{T_i}\sum_{\ell=1}^{T_i}
\ell_{i,\ell}(\theta;A_i)
\right]
-\beta_{\mathrm{KL}}\mathcal K(\theta),
\end{equation}
where the expectation is over tasks and sampled trajectories,
and $\varepsilon$ is the ratio-clipping radius.
The term $\mathcal K(\theta)$ denotes the sampled KL penalty
against a fixed reference policy $\pi_{\mathrm{ref}}$,
with its strength controlled by $\beta_{\mathrm{KL}}\geq0$.
Group normalization provides the baseline without requiring
a learned value function.

\paragraph{On-Policy Self-Distillation (OPSD).}
OPSD \citep{zhao2026opsd} uses the same LLM as student and teacher under different conditioning contexts. For a prompt $x$ with reference information $h$, the student samples a response $y$ from $\student(\cdot\mid x)$, while the teacher additionally receives $h$. A basic distillation objective is
\begin{equation}
\label{eq:opsd}
\mathcal L_{\mathrm{OPSD}}(\theta)=\mathbb E_{x,h,y}\!\left[\frac1{|y|}\sum_{\ell=1}^{|y|}D\!\left(\pi_\theta(\cdot\mid c_\ell\oplus h)\,\|\,\student(\cdot\mid c_\ell)\right)\right],
\end{equation}
where $c_\ell=(x,y_{<\ell})$, $c_\ell\oplus h=(x\oplus h,y_{<\ell})$ augments the prompt with reference information while retaining the sampled prefix, and $D$ is a distributional divergence. The sampled prefixes and teacher distributions are held fixed when differentiating this objective, hence gradients flow only through the student predictions. Evaluating student-generated prefixes supplies local supervision without an external teacher model.

With terminal rewards, GRPO assigns the same trajectory
advantage in Eq.~\eqref{eq:grpo-adv} to every generated token.
This shared signal cannot distinguish useful intermediate
decisions from local mistakes within the same rollout.
OPSD provides an additional source of supervision by
conditioning the teacher on a reference solution.
The teacher can provide different guidance at different
positions in the student's response, supplying local
information that the terminal reward alone does not provide.
Inspired by this idea, \method{} uses successful peer
rollouts to obtain teacher feedback and aggregates it
within interaction segments.
This feedback adjusts how strongly each segment is
reinforced or penalized, offering an approach to addressing the credit-assignment problem under sparse terminal rewards.

\subsection{\method{}}
\label{sec:algorithm}

After rollout collection and reward evaluation, \method{} uses teacher feedback to reweight the GRPO base advantage at the segment level.
The adjusted advantages are then used in the objective in Eq.~\eqref{eq:grpo-obj}, retaining its clipped surrogate and KL regularization.
Figure~\ref{fig:pipeline} illustrates the overall training procedure.

\begin{figure}[t]
\centering
\includegraphics[width=\linewidth]{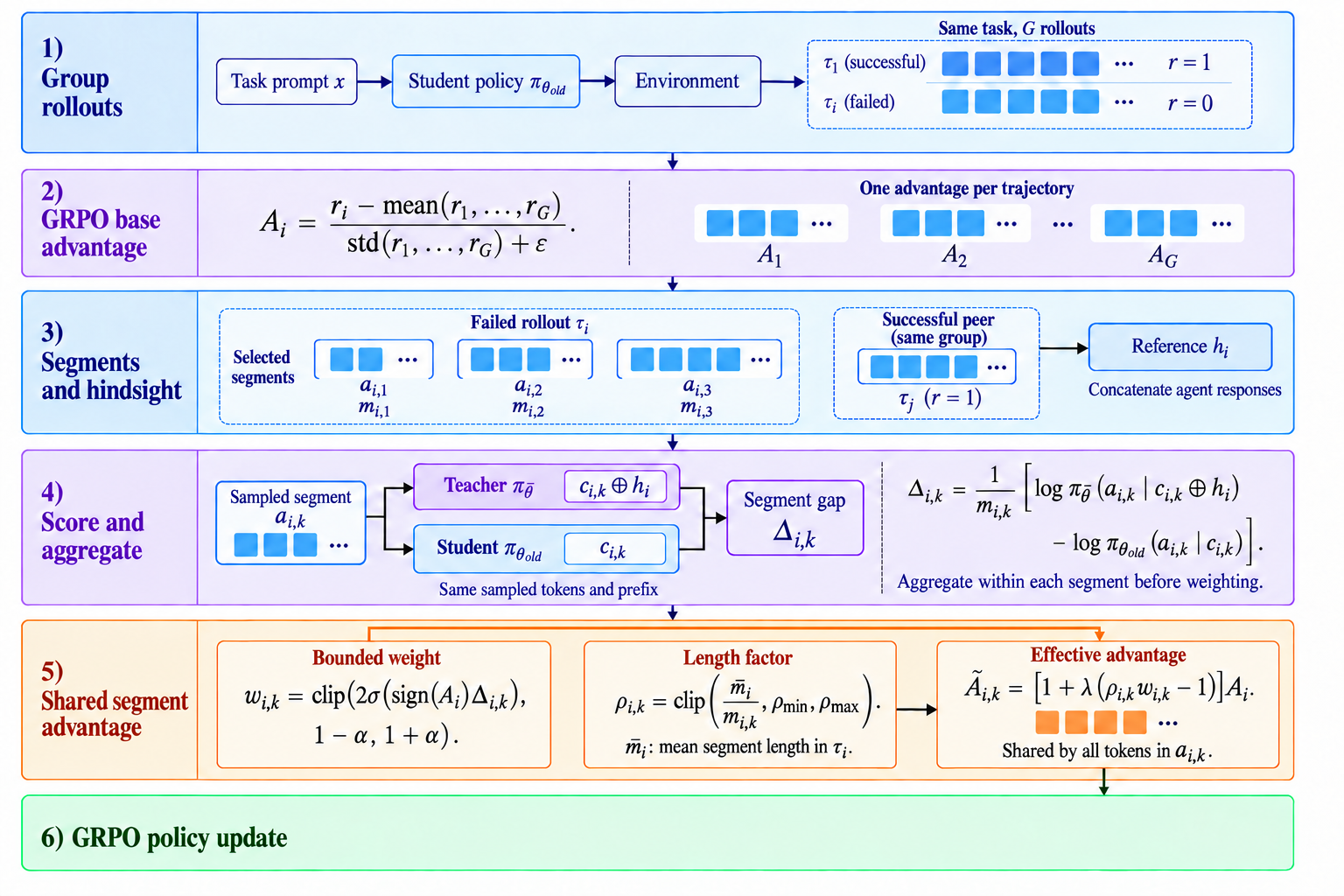}
\caption{\textbf{\method{} Framework.} A successful peer supplies context for a policy-copy teacher. A length-normalized likelihood comparison of each selected segment produces a bounded weight. Length balancing and interpolation with the base advantage then determine the segment's effective advantage for the policy update.}
\label{fig:pipeline}
\end{figure}

\paragraph{Step 1: identify interaction segments.}
An interaction segment is a contiguous span of the agent's response
expressing a coherent unit of behavior, such as an action enclosed in
\texttt{<action>...</action>} or a reasoning passage enclosed in
\texttt{<thinking>...</thinking>}.
For the $i$-th trajectory $\tau_i$, we select $K_i$ nonempty, disjoint segments
for credit assignment.
Let $a_{i,k}$ denote the $k$-th selected segment and
$m_{i,k}=|a_{i,k}|$ its length.
We denote by $c_{i,k}$ the context available to the agent immediately
before generating $a_{i,k}$.
It consists of the task prompt $x$ and all preceding agent-generated
content and environment observations, including any content earlier
in the same response.
We assess and assign credit to each selected segment as a whole.

\paragraph{Step 2: construct a teacher context from successful peers.}
Within each group, we select one successful rollout and construct a reference text by concatenating its agent
responses in interaction order, separated by step labels.
For example, a successful shopping rollout may provide responses
covering product search, option selection, and purchase.
All lower-reward trajectories in the group share this reference text.
We denote the reference used to evaluate trajectory $\tau_i$ by $h_i$. The teacher $\teacher$ is a copy of the trainable student policy
$\student$. Its parameters $\bar\theta$ are copied from the student at each teacher refresh and held fixed until the next refresh.

\paragraph{Step 3: assess each segment.}
The teacher assesses the
student's sampled behavior with additional reference information.
We write $c_{i,k}\oplus h_i$ for the teacher's context, which augments
the original interaction context $c_{i,k}$ with the successful
reference $h_i$.
The student receives only $c_{i,k}$.
Their teacher--student log-probability gap is
\begin{equation}
\label{eq:evidence}
\Delta_{i,k}=\frac1{m_{i,k}}\left[
\log\teacher(a_{i,k}\mid c_{i,k}\oplus h_i)
-\log\student(a_{i,k}\mid c_{i,k})\right].
\end{equation}
The student likelihood is evaluated at the rollout parameters
$\theta=\theta_{\mathrm{old}}$, and the gap is held fixed during
the subsequent policy update.
A positive gap indicates stronger support for the segment from
the reference-informed teacher than from the student, while a
negative gap indicates weaker support.
Such feedback can vary across segments within a failed trajectory,
even though they share the same outcome-derived advantage.
We therefore use the gap to determine each segment's weight in
Step~4, adjusting its contribution to the parameter update according
to the teacher's local feedback.

\paragraph{Step 4: compute segment weights.}
With the sigmoid function $\sigma(z)=(1+e^{-z})^{-1}$, we map the teacher--student gap to a bounded weight:
\begin{equation}
\label{eq:gate}
w_{i,k}=\clip\!\left(2\sigma\!\left(\sign(A_i)\Delta_{i,k}\right),1-\alpha,1+\alpha\right),
\qquad 0<\alpha<1.
\end{equation}
The sigmoid function smoothly converts the signed gap into a positive weight, and the factor of $2$ ensures that zero gap $\Delta_{i, k} = 0$ gives a weight of $1$. To balance the shaped coefficients across segment lengths, we use
\begin{equation}
\label{eq:rho}
\rho_{i,k}=\clip\!\left(\frac{\bar m_i}{m_{i,k}},\rho_{\min},\rho_{\max}\right),
\qquad \bar m_i=\frac1{K_i}\sum_{k=1}^{K_i}m_{i,k},
\end{equation}
where $0<\rho_{\min}\le\rho_{\max}$.
This factor compensates for differences in segment length,
reducing the influence of length on the total credit
assigned to each segment.

\paragraph{Step 5: update the policy.}
We interpolate between the base and reweighted advantages:
\begin{equation}
\label{eq:aeff}
g_{i,k}=1+\lambda(\rho_{i,k}w_{i,k}-1),\qquad
\widetilde A_{i,k}=g_{i,k}A_i,\qquad 0\le\lambda\le1.
\end{equation}
We then update the student parameters $\theta$ by maximizing
the objective in Eq.~\eqref{eq:grpo-obj} with these
segment-level advantages $\widetilde A_{i,k}$.

\subsection{Understanding Segment-Level Credit Assignment}
\label{sec:theory}

We use an ideal success-conditioned reference to explain how hindsight feedback refines credit for individual decisions. Let $u(z)=\clip(2\sigma(z)-1,-\alpha,\alpha)$, a bounded transform that preserves the sign of its input. Its odd symmetry gives
\begin{equation}
\label{eq:credit-decomposition}
\widetilde A_{i,k}=\underbrace{A_i}_{\text{outcome}}
+\underbrace{\lambda A_i(\rho_{i,k}-1)}_{\text{length adjustment}}
+\underbrace{\lambda\rho_{i,k}|A_i|u(\Delta_{i,k})}_{\text{hindsight adjustment}}.
\end{equation}
For a fixed segment, we suppress $(i,k)$ and write $B=(1-\lambda+\lambda\rho)A$ for the coefficient without hindsight. The remaining term, $\lambda\rho|A|u(\Delta)$, raises or lowers this coefficient according to the segment's feedback.

To judge this adjustment, let $\pi=\pi_{\theta_{\mathrm{old}}}$ and consider a segment $a$ of length $m>0$ at context $c$. Define $Q^\pi(c,a)=\Pr_\pi(R=1\mid c,a)$ as the success probability after choosing $a$ and then following $\pi$, where $R=1$ denotes terminal success. The baseline $V^\pi(c)=\Pr_\pi(R=1\mid c)$ is the success probability conditioned on previous context $c$. Their difference, $D^\pi(c,a)=Q^\pi(c,a)-V^\pi(c)$, is the segment's local value. Thus $D^\pi>0$ means choosing $a$ improves success probability, and $D^\pi<0$ means it reduces it.

We now consider the ideal reference $q^\star(a\mid c)=\Pr_\pi(a\mid c,R=1)$. Successful-peer conditioning aims to approximate this assessment for the failed rollouts selected in Step~2. We now use Proposition \ref{prop:hindsight-credit} to explain why the hindsight information is useful for refined credit assignment.

\begin{proposition}
\label{prop:hindsight-credit}
Assume $V^\pi(c)>0$ and $\pi(a\mid c)>0$. Let $\Delta^\star=m^{-1}[\log q^\star(a\mid c)-\log\pi(a\mid c)]$ and $\widetilde A^\star=B+\lambda\rho|A|u(\Delta^\star)$.
\begin{enumerate}[label=(\alph*),leftmargin=*,topsep=3pt,itemsep=3pt,parsep=0pt]
\item The reference gap satisfies
\begin{equation}
\label{eq:hindsight-value}
\frac{q^\star(a\mid c)}{\pi(a\mid c)}=\frac{Q^\pi(c,a)}{V^\pi(c)},
\qquad \sign u(\Delta^\star)=\sign D^\pi(c,a).
\end{equation}
\item For any $\rho>0$, the refined coefficient satisfies
\begin{equation}
\label{eq:conflict-reduction}
[-\widetilde A^\star D^\pi(c,a)]_+\leq[-B D^\pi(c,a)]_+,
\end{equation}
where $[z]_+=\max\{z,0\}$.
\end{enumerate}
\end{proposition}

In Part (a), Equation~\eqref{eq:hindsight-value} shows that, under the ideal reference, $u(\Delta^\star)$ and the segment's local value $D^\pi(c,a)$ have the same direction. Therefore, scaling this signal $u(\Delta^\star)$ by $\lambda\rho|A|$, the hindsight adjustment (cf. Eq.~\eqref{eq:credit-decomposition}) also aligns with the direction of the local value $D^\pi(c,a)$. This means a positive hindsight adjustment identifies a segment that increases the final success probability, while a negative hindsight adjustment identifies one that decreases it. This alignment explains why we use the teacher--student gap as hindsight feedback. After length normalization and the sign-preserving transform $u$, the ideal signal $\Delta^\star$ indicates whether each segment helps or harms task success. This local information helps refine credit assignment within the trajectory.

In Part (b), $B$ stands for the segment's advantage adjustment without hindsight feedback. When we choose the length normalization term $\rho = 1$, $B$ becomes the base GRPO advantage $A$. The quantity $[-BD^\pi(c,a)]_+$ measures the conflict between this coefficient $B$ and the segment's local value $D^{\pi}$. It is positive when $B$ opposes $D^{\pi}$ and zero otherwise. When a useful decision in a failed rollout has $D^\pi(c,a)>0$ but receives the advantage $B<0$, we have $[-BD^\pi(c,a)]_+>0$. This indicates a credit-assignment error, as a useful decision receives negative credit because its trajectory fails. Equation~\eqref{eq:conflict-reduction} shows that ideal hindsight does not worsen this error for any $\rho>0$.

\section{Experiments}
\label{sec:experiments}

We conduct experiments on ALFWorld and WebShop to demonstrate the effectiveness of \method{} through comparisons with existing RL baselines.

\subsection{Setup}
\label{sec:setup-exp}

We fine-tune Qwen2.5-7B-Instruct \citep{qwen2024qwen25} on two interactive benchmarks, ALFWorld \citep{shridhar2021alfworld} and WebShop \citep{yao2022webshop}.
ALFWorld covers six families of household tasks.
WebShop requires agents to search for and purchase products that satisfy user requests.
Both environments use binary terminal success rewards for training.

All methods share the same backbone, prompts, environments, reward function, optimizer, rollout settings, and evaluation protocol.
We train each configuration for $200$ optimization steps with $16$ tasks per batch, a learning rate of $5\times10^{-7}$, and a KL coefficient of $\beta_{\mathrm{KL}}=0.01$.
For \method{}, we set $\alpha=0.2$, $\rho_{\min}=0.25$, and $\rho_{\max}=4$.
We evaluate two mixing weights, $\lambda\in\{0.2,0.05\}$, each held fixed throughout training.

Our primary metric is the final validation success rate at step $200$, evaluated on $128$ held-out tasks.
We report the mean and sample standard deviation over three independent runs.
To examine performance in more detail, we also report subtask success rates by ALFWorld task family and the continuous task score in WebShop, which measures how well the purchased product matches the request.

We compare our \method{} with GRPO \citep{shao2024deepseekmath}, which assigns a single outcome-derived advantage to all generated tokens within each trajectory.
We also consider baselines that combine RL with self-distillation:
RLSD \citep{yang2026rlsd} uses self-distillation gaps to reweight advantages at the token level, while SDAR \citep{lu2026sdar} applies token-level gating to an auxiliary distillation loss.
GRPO+OPSD augments the GRPO objective with the OPSD distillation loss \citep{zhao2026opsd}.
Finally, StepOPSD \citep{zhang2026stepopsd} uses hindsight feedback to compute token-specific advantage weights within selected interaction spans.

\subsection{Main results}
\label{sec:results}

Table~\ref{tab:main-results} reports the overall comparison. \method{} achieves the highest mean final validation success rate on both benchmarks at both tested mixing weights.
Relative to GRPO, the gains are $13.80$--$14.32$ percentage points on ALFWorld and $8.85$--$9.11$ points on WebShop.
Compared with StepOPSD at the corresponding reported $\lambda$ values, \method{} gains $6.77$ points on ALFWorld at both settings and $5.99$--$7.55$ points on WebShop.
These results support the effectiveness of using segment-level self-distillation feedback to refine GRPO credit assignment.

\begin{table}[t]
\hypersetup{hidelinks}
\centering
\setlength{\abovecaptionskip}{7pt}
\setlength{\belowcaptionskip}{0pt}
{\small
\setlength{\tabcolsep}{4pt}
\renewcommand{\arraystretch}{1.20}
\begin{tabular*}{\linewidth}{@{\extracolsep{\fill}}lcrcr@{}}
\toprule
& \multicolumn{2}{c}{ALFWorld} & \multicolumn{2}{c}{WebShop} \\
\cmidrule(lr){2-3}\cmidrule(l){4-5}
Method & Success rate (\%) & Gain & Success rate (\%) & Gain \\
\midrule
GRPO & $70.57 \pm 7.83$ & -- & $66.15 \pm 1.19$ & -- \\
SDAR & $67.19 \pm 9.50$ & $-3.39$ & $61.72 \pm 6.10$ & $-4.43$ \\
RLSD & $73.96 \pm 5.76$ & $+3.39$ & $69.01 \pm 6.31$ & $+2.86$ \\
GRPO+OPSD & $72.92 \pm 4.51$ & $+2.34$ & $69.01 \pm 3.16$ & $+2.86$ \\
StepOPSD ($\lambda=0.2$) & $77.60 \pm 2.96$ & $+7.03$ & $69.27 \pm 1.63$ & $+3.12$ \\
StepOPSD ($\lambda=0.05$) & $78.12 \pm 3.41$ & $+7.55$ & $67.45 \pm 4.01$ & $+1.30$ \\
\method{} ($\lambda=0.2$) & \textcolor{saBlue}{$84.38 \pm 2.82$} & \textcolor{saBlue}{$+13.80$} & \textcolor{saRed}{$75.26 \pm 1.19$} & \textcolor{saRed}{$+9.11$} \\
\method{} ($\lambda=0.05$) & \textcolor{saRed}{$84.90 \pm 1.19$} & \textcolor{saRed}{$+14.32$} & \textcolor{saBlue}{$75.00 \pm 1.35$} & \textcolor{saBlue}{$+8.85$} \\
\bottomrule
\end{tabular*}
\par}
\caption{\textbf{Final validation success rate (\%)} at step 200 with Qwen2.5-7B-Instruct, mean $\pm$ sample standard deviation over three independent runs. Gain is the change in mean success rate relative to GRPO, in percentage points. \textcolor{saRed}{Red} and \textcolor{saBlue}{blue} indicate the best and second-best mean results in each column, respectively.}
\label{tab:main-results}
\end{table}

The improvements persist across both tested mixing weights.
Although $\lambda=0.2$ and $\lambda=0.05$ differ by a factor of four, their mean success rates differ by only $0.52$ percentage points on ALFWorld and $0.26$ points on WebShop.
Across the three independent runs, \method{} has lower variability than GRPO on ALFWorld, with standard deviations of $2.82$ and $1.19$ points compared with $7.83$.
On WebShop, its standard deviations of $1.19$ and $1.35$ points are comparable to GRPO's $1.19$.
The higher mean success rates therefore accompany lower run-to-run variability on ALFWorld and similar variability on WebShop.

\subsection{Fine-grained analysis}
\label{sec:finegrained}

\paragraph{Subtask performance.}
\label{sec:task-outcomes}

Table~\ref{tab:alfworld-subtasks} shows that \method{} performs
strongly across the six ALFWorld task families.
At both tested mixing weights $\lambda = 0.2$ and $\lambda = 0.05$, it outperforms all evaluated
baselines on Pick, Clean, Cool, and Pick2, and remains within
$0.30$ percentage points of the best baseline on Look.
These broad gains persist across the two mixing weights,
which differ by a factor of four.

The improvements extend to complicated tasks such as Heat, Cool, and Pick2, which requires object-state changes
or the handling of multiple objects.
On Pick2, \method{} exceeds the strongest baseline by $7.02$
and $10.10$ percentage points at $\lambda=0.2$ and $0.05$,
respectively.
On Cool, $\lambda=0.2$ improves over the strongest baseline
by $6.89$ percentage points.
On Heat, $\lambda=0.05$ exceeds SDAR, the strongest baseline
for this task, by $6.67$ percentage points.
These tasks require several coordinated decisions, and a final
failure can follow both useful intermediate actions and local
mistakes.
The gains are consistent with the motivation for segment-level
credit assignment, which uses successful-peer feedback to
differentiate the training weights of intermediate decisions.

\begin{table}[!htbp]
\hypersetup{hidelinks}
\centering
\setlength{\abovecaptionskip}{7pt}
\setlength{\belowcaptionskip}{0pt}
{\small
\setlength{\tabcolsep}{4pt}
\renewcommand{\arraystretch}{1.20}
\begin{tabular*}{\linewidth}{@{\extracolsep{\fill}}lrrrrrrr@{}}
\toprule
Method & Pick & Look & Clean & Heat & Cool & Pick2 & All \\
\midrule
GRPO & $89.22$ & $73.33$ & $81.39$ & $61.90$ & $58.21$ & $60.06$ & $70.57$ \\
SDAR & $80.94$ & $82.22$ & $76.55$ & \textcolor{saBlue}{$71.43$} & $54.68$ & $48.00$ & $67.19$ \\
RLSD & $89.65$ & $74.44$ & $88.49$ & $55.24$ & $55.04$ & $70.71$ & $73.96$ \\
GRPO+OPSD & $90.95$ & $81.72$ & $77.92$ & $39.05$ & $62.29$ & $66.04$ & $72.92$ \\
StepOPSD ($\lambda=0.2$) & $90.96$ & $83.33$ & $83.50$ & $52.86$ & $68.46$ & $75.38$ & $77.60$ \\
StepOPSD ($\lambda=0.05$) & $90.32$ & \textcolor{saRed}{$85.86$} & $88.60$ & $69.52$ & $65.70$ & $64.67$ & $78.12$ \\
\method{} ($\lambda=0.2$) & \textcolor{saBlue}{$93.73$} & \textcolor{saBlue}{$85.56$} & \textcolor{saBlue}{$91.49$} & $61.67$ & \textcolor{saRed}{$75.35$} & \textcolor{saBlue}{$82.40$} & \textcolor{saBlue}{$84.38$} \\
\method{} ($\lambda=0.05$) & \textcolor{saRed}{$96.00$} & \textcolor{saBlue}{$85.56$} & \textcolor{saRed}{$92.68$} & \textcolor{saRed}{$78.10$} & \textcolor{saBlue}{$68.82$} & \textcolor{saRed}{$85.48$} & \textcolor{saRed}{$84.90$} \\
\bottomrule
\end{tabular*}
\par}
\caption{\textbf{ALFWorld success rate (\%) by task family} at step 200, mean over three independent runs. Family rates are averaged separately; \emph{All} reports the mean overall validation success rate from Table~\ref{tab:main-results}. \textcolor{saRed}{Red} and \textcolor{saBlue}{blue} indicate the best and second-best mean results in each column, respectively.}
\label{tab:alfworld-subtasks}
\end{table}

\paragraph{Product-matching quality in WebShop.}

Success rate measures whether a purchase fully satisfies the shopping request.
WebShop's task score provides a finer-grained assessment of how well the purchase meets that request, giving credit to partial matches as well as complete success \citep[Section~3.1]{yao2022webshop}. The score of one evaluation episode is
\begin{equation}
\label{eq:webshop-score}
S=r_{\mathrm{type}}\,
\frac{|U_{\mathrm{att}}\cap Y_{\mathrm{att}}|+|U_{\mathrm{opt}}\cap Y_{\mathrm{opt}}|
+\mathbf{1}\{y_{\mathrm{price}}\leq u_{\mathrm{price}}\}}
{|U_{\mathrm{att}}|+|U_{\mathrm{opt}}|+1}.
\end{equation}
Here $U_{\mathrm{att}}$ and $U_{\mathrm{opt}}$ are the requested attributes and options respectively. $Y_{\mathrm{att}}$ and $Y_{\mathrm{opt}}$ describe the purchased product and selected options respectively. The intersections count matches under the benchmark's rules. The product price $y_{\mathrm{price}}$ is compared with the budget $u_{\mathrm{price}}$, and $r_{\mathrm{type}}\in[0,1]$ measures product-type matching. Episodes without a purchase receive zero. At training step $200$, we average $S$ over all $128$ validation episodes to obtain each run's task score. We then report the mean of these scores on a $[0,1]$ scale. Task score captures how well a purchase satisfies the request, even when success rate counts it as a failure. For example, a shirt may match the requested material, color, and budget but have the wrong size. A second shirt may meet the same budget but have the wrong size, material, and color. Both purchases count as failures, but task score credits the additional requirements satisfied by the first.

Table~\ref{tab:task-quality} shows gains over StepOPSD of $0.0591$ and $0.0212$ at $\lambda=0.2$ and $0.05$, respectively. Both settings also exceed GRPO. Thus, higher success rates are accompanied by better average satisfaction of the shopping requirements. Unlike StepOPSD’s token-specific weights, our method assigns a shared hindsight-derived weight to each complete interaction segment. The higher task scores are consistent with evaluating search, option-selection, and purchase decisions at this shared granularity.

\begin{table}[!htbp]
\hypersetup{hidelinks}
\centering
\begin{minipage}{0.60\linewidth}
\setlength{\abovecaptionskip}{7pt}
\setlength{\belowcaptionskip}{0pt}
{\small
\setlength{\tabcolsep}{6pt}
\renewcommand{\arraystretch}{1.20}
\begin{tabular*}{\linewidth}{@{\extracolsep{\fill}}lr@{}}
\toprule
Method & Task score\\
\midrule
GRPO & $0.7781$\\
StepOPSD ($\lambda=0.2$) & $0.8231$\\
StepOPSD ($\lambda=0.05$) & $0.8268$\\
\method{} ($\lambda=0.2$) & \textcolor{saRed}{$0.8822$}\\
\method{} ($\lambda=0.05$) & \textcolor{saBlue}{$0.8480$}\\
\bottomrule
\end{tabular*}
\par}
\caption{\textbf{WebShop task score.} Higher is better. \textcolor{saRed}{Red} and \textcolor{saBlue}{blue} mark the best and second-best results among the listed methods.}
\label{tab:task-quality}
\end{minipage}
\end{table}

\section{Related Work}
\label{sec:related}

Outcome-supervised RL has been applied to reasoning models \citep{shao2024deepseekmath,deepseek2025r1} and multi-turn language agents \citep{jin2025search,wang2025ragen}. Actor-critic implementations of PPO \citep{schulman2017proximal} commonly pair a learned value function with generalized advantage estimation (GAE) \citep{schulman2016gae}, which combines temporal-difference residuals across time in a construction related to eligibility traces. ArCHer treats utterances as high-level actions and uses off-policy critic learning to obtain utterance-level advantages for token-level policy updates \citep{zhou2024archer}. Other approaches use recurring environment states \citep{feng2025gigpo}, Monte Carlo resampling \citep{kazemnejad2025vineppo,guo2025spo}, or process reward models \citep{lightman2023let,wang2024mathshepherd}. POAD studies action-level optimization \citep{wen2024poad}. TREK improves exploration on hard prompts through distillation of verified proposal trajectories followed by GRPO refinement \citep{xu2026trek}. TRIAGE uses an LLM judge to assign semantic roles to environment-facing segments, then adds predefined role-specific rewards to their GRPO advantages \citep{xu2026triage}. In comparison, our method \method{} retains GRPO's critic-free trajectory advantage and obtains local feedback by scoring sampled segments under student and successful-peer-conditioned teacher policies. The resulting log-probability gaps determine bounded weights shared within each selected segment, allowing the adjustment to depend on its hindsight feedback without predefined role categories or role-specific reward constants. This provides segment-specific credit without fitting a critic to expected task returns or using temporal-difference backups, external judges, or additional environment rollouts.

This use of successful-peer context is related to context distillation, which transfers behavior induced by privileged context \citep{snell2022context}. On-policy distillation evaluates the student's own samples \citep{agarwal2024onpolicy}. OPSD uses the same model under different conditioning contexts \citep{zhao2026opsd}. RLSD applies self-distillation gaps to token-level advantage magnitudes \citep{yang2026rlsd}; SDAR gates an auxiliary distillation objective \citep{lu2026sdar}; and StepOPSD uses successful-peer hindsight for token-specific weighting within action-centered spans \citep{zhang2026stepopsd}. \method{} instead aggregates feedback over each selected segment before weighting its GRPO advantage. Concurrent AgentOPSD aggregates feedback over full turns using a retrieved skill and a recursive belief update \citep{wang2026agentopsd}; our method uses independently scored environment-facing segments and successful peers from the current group.

\section{Conclusion}
\label{sec:conclusion}

In long-horizon tasks, RL agents often receive rewards only when an episode ends. These sparse rewards provide little direct feedback on which intermediate decisions helped or harmed task completion. \method{} addresses this credit-assignment problem through successful-peer hindsight. It aggregates teacher--student log-probability gaps within complete interaction segments before assigning each segment a shared, bounded advantage weight, providing differentiated feedback without a learned critic.

Our analysis explains how the hindsight correction aligns with local decision value under an ideal success-conditioned reference. On ALFWorld and WebShop, \method{} achieves higher mean final success than the baselines at both tested mixing weights. These results support supplementing sparse outcome supervision with segment-level hindsight feedback for long-horizon agent learning.

\bibliographystyle{unsrtnat}
\bibliography{references}

\appendix
\section{Theory of segment-level credit refinement}
\label{app:credit-theory}

We use the notation of Section~\ref{sec:theory}. Sampled trajectories, base advantages, teacher feedback, and segment weights are held fixed during each policy update. Unless indices are needed, we analyze one selected segment of length $m>0$, with $\pi=\pi_{\theta_{\mathrm{old}}}$.

\subsection{Value-aligned credit correction}
\label{app:hindsight}

\begin{proof}[Proof of Proposition~\ref{prop:hindsight-credit}]
The function $u(z)=\clip(2\sigma(z)-1,-\alpha,\alpha)$ is odd, preserves sign, and satisfies $|u(z)|\leq\alpha<1$. Thus
\[
w-1=u(\sign(A)\Delta),\qquad A(w-1)=|A|u(\Delta).
\]
Substituting these identities into Eq.~\eqref{eq:aeff} gives
\[
\widetilde A=(1-\lambda+\lambda\rho)A+\lambda\rho|A|u(\Delta),
\]
which proves Eq.~\eqref{eq:credit-decomposition}.

For the success-conditioned reference, Bayes' rule gives
\[
q^\star(a\mid c)
=\frac{\Pr_\pi(a,R=1\mid c)}{\Pr_\pi(R=1\mid c)}
=\pi(a\mid c)\frac{Q^\pi(c,a)}{V^\pi(c)}.
\]
When $Q^\pi(c,a)>0$, monotonicity of the logarithm implies
\[
\sign\Delta^\star
=\sign\!\left(\log\frac{Q^\pi(c,a)}{V^\pi(c)}\right)
=\sign D^\pi(c,a).
\]
Since $u$ preserves sign, Eq.~\eqref{eq:hindsight-value} follows. When $Q^\pi(c,a)=0$, the same conclusion holds with $u(\Delta^\star)=-\alpha$. The correction $\lambda\rho|A|u(\Delta^\star)$ consequently has the direction of $D^\pi(c,a)$ whenever it is nonzero.

For part (b), write $D=D^\pi(c,a)$. Part (a) gives $u(\Delta^\star)D\geq0$. For any $\rho>0$ and $0\leq\lambda\leq1$,
\[
(\widetilde A^\star-B)D
=\lambda\rho|A|u(\Delta^\star)D\geq0.
\]
Thus $-\widetilde A^\star D\leq-BD$. Since $z\mapsto[z]_+$ is nondecreasing,
\[
[-\widetilde A^\star D]_+\leq[-BD]_+,
\]
which proves Eq.~\eqref{eq:conflict-reduction}, including the cases $A=0$, $D=0$, or $\lambda=0$.

To see how this correction acts within a failed rollout, take $A<0$ and $\lambda>0$. Then $B<0$, and odd symmetry gives
\[
\widetilde A^\star
=A\bigl[1-\lambda+\lambda\rho(1-u(\Delta^\star))\bigr]<0,
\]
since $1-u(\Delta^\star)\geq1-\alpha>0$. If $D>0$, part (a) gives a positive correction, so $B<\widetilde A^\star<0$. If $D<0$, the correction is negative, so $\widetilde A^\star<B<0$. These relations establish the weaker and stronger penalties described in the main text at the same length factor.
\end{proof}

\subsection{Computing segment feedback}
\label{app:segment-equivalence}

Let $a_{i,k}=(z_1,\ldots,z_m)$ be a sampled segment with $m=m_{i,k}>0$. Assume both policies assign positive likelihood to it under their respective contexts. The teacher and student use the same sampled prefix $z_{<j}$, with the student evaluated at $\theta_{\mathrm{old}}$ and the teacher additionally conditioned on $h_i$. By autoregressive factorization, the length-normalized segment gap in Eq.~\eqref{eq:evidence} equals the average of its conditional token gaps:
\begin{equation}
\label{eq:segment-evidence-sum}
\begin{split}
\Delta_{i,k}&=\frac1m\sum_{j=1}^m\delta_j,\\
\delta_j&=\log\teacher(z_j\mid c_{i,k}\oplus h_i,z_{<j})
-\log\pi_{\theta_{\mathrm{old}}}(z_j\mid c_{i,k},z_{<j}).
\end{split}
\end{equation}

\subsection{Length-balancing rationale}
\label{app:proof-mass}

We compare coefficient totals over $K_i>0$ nonempty selected segments within one trajectory, omitting their common normalization $1/T_i$. Equation~\eqref{eq:aeff} gives $\widetilde A_{i,k}=(1-\lambda)A_i+\lambda\rho_{i,k}w_{i,k}A_i$, with $0\leq\lambda\leq1$. Without length balancing, the reweighted component accumulates absolute coefficient $\lambda m_{i,k}w_{i,k}|A_i|$, which contains an explicit factor of segment length. Suppose all length factors are unclipped, so $\rho_{i,k}=\bar m_i/m_{i,k}$ with $\bar m_i=K_i^{-1}\sum_{k=1}^{K_i}m_{i,k}$. Then
\begin{equation}
\label{eq:segment-length-balance}
\lambda m_{i,k}\rho_{i,k}w_{i,k}|A_i|
=\lambda\bar m_iw_{i,k}|A_i|.
\end{equation}
Thus, equally weighted segments in the same trajectory receive the same accumulated coefficient in the reweighted component, regardless of length. The base component $(1-\lambda)A_i$ remains unchanged. This balances coefficients in the reweighted component, not the full update or segment gradient norms.

The numerator $\bar m_i$ preserves the overall coefficient scale. With $\bar w_i=K_i^{-1}\sum_{k=1}^{K_i}w_{i,k}$, substitution into Eq.~\eqref{eq:aeff} and $\sum_k m_{i,k}=K_i\bar m_i$ give
\begin{equation}
\label{eq:mass}
\sum_{k=1}^{K_i}m_{i,k}\widetilde A_{i,k}
=\bigl[1+\lambda(\bar w_i-1)\bigr]
\sum_{k=1}^{K_i}m_{i,k}A_i.
\end{equation}
When all $w_{i,k}=1$, the total advantage coefficient over selected segments is unchanged. Weight clipping in Eq.~\eqref{eq:gate} remains in effect, so $w_{i,k}\in[1-\alpha,1+\alpha]$ and the common multiplier lies in $[1-\lambda\alpha,1+\lambda\alpha]$. Compared with using the unclipped factor $1/m_{i,k}$ alone, the numerator $\bar m_i$ restores the reweighted component's overall coefficient scale.

If a length factor is clipped, Eqs.~\eqref{eq:segment-length-balance} and~\eqref{eq:mass} need not hold exactly. Under the parameter ranges in Eqs.~\eqref{eq:gate}--\eqref{eq:aeff}, $\rho_{i,k}$ and $w_{i,k}$ remain positive and bounded. Hence $g_{i,k}$ is positive and bounded, and $\widetilde A_{i,k}$ preserves the sign of $A_i$, including $\widetilde A_{i,k}=0$ when $A_i=0$.

\end{document}